\documentclass[review,onefignum,onetabnum]{siamonline171218}

\usepackage{booktabs} % For formal tables
\usepackage{amsmath}
\usepackage{amssymb}
\usepackage{multirow}
\usepackage{multicol}
\usepackage{soul}

\usepackage{mathtools, cuted}

\DeclareMathOperator*{\argmin}{arg\,min}

\usepackage{tikz}
\usetikzlibrary{shapes.geometric, arrows}
\tikzstyle{normal} = [rectangle, rounded corners, minimum width=1.0cm, minimum height=0.8cm,text centered, draw=black]
\tikzstyle{arrow} = [thick,->,>=stealth]

\usepackage{lipsum}
\usepackage{amsfonts}
\usepackage{graphicx}
\usepackage{epstopdf}

\usepackage{algorithm}
\usepackage[noend]{algpseudocode}

\ifpdf
  \DeclareGraphicsExtensions{.eps,.pdf,.png,.jpg}
\else
  \DeclareGraphicsExtensions{.eps}
\fi

\usepackage{enumitem}
\setlist[enumerate]{leftmargin=.5in}
\setlist[itemize]{leftmargin=.5in}

\newsiamremark{remark}{Remark}
\newsiamremark{hypothesis}{Hypothesis}
\crefname{hypothesis}{Hypothesis}{Hypotheses}
\newsiamthm{claim}{Claim}

\headers{Mathematical Analysis of Adversarial Attacks}{Z. Dou, and S. Osher, and B. Wang}

\title{Mathematical Analysis of Adversarial Attacks\thanks{Submitted to the editors DATE.
}}

\author{Zehao Dou\thanks{School of Mathematical Science,
       Peking University, Beijing, China
  (\email{zehaodou@pku.edu.cn}).
  }
\and Stanley J. Osher and Bao Wang\thanks{Department of Mathematics, UCLA, Los Angeles, CA, 90095-1555
  (\email{sjo@math.ucla.edu}, \email{wangbaonj@gmail.com}). 
  }
}

\usepackage{amsopn}

\ifpdf
\hypersetup{
  pdftitle={Mathematical Analysis of Adversarial Attacks},
  pdfauthor={Z. Dou, S. Osher and B. Wang}
}
\fi

\begin{document}

\maketitle

% REQUIRED
\begin{abstract}
In this paper, we analyze the efficacy of the fast gradient sign method (FGSM) and the Carlini-Wagner's $L_2$ (CW-L2) attack. We prove that, within a specific regime, the untargeted FGSM can fool any convolutional neural nets (CNN) with ReLU activation; the targeted FGSM can mislead any CNN with ReLU activation to classify any given image into any prescribed class. For a particular two-layer neural nets, a linear layer followed by the softmax output activation, we show that the CW-L2 attack increases the ratio of the classification probability between the target and the ground truth classes. Moreover, we provide numerical results to verify our theoretical results.
\end{abstract}

% REQUIRED
% REQUIRED  % TODO: Double check this

\section{Introduction}
The adversarial vulnerability \cite{Szegedy:2013} of deep neural nets (DNN) threatens their applicability in security critical tasks, e.g., autonomous cars \cite{Akhtar:2018}, robotics \cite{Giusti:2016Drones}, DNN-based malware detection systems
 \cite{PapernotSecurity:2016,PapernotMalware:2016}. Since the pioneering work by Szegedy et al. \cite{Szegedy:2013}, many advanced adversarial attack schemes have been devised to generate imperceptible perturbations to sufficiently fool the DNN \cite{Goodfellow:2014AdversarialTraining,PapernotAttack:2016,CWAttack:2016}. Not only are adversarial attacks successful in white-box attacks, i.e., when the adversary has access to the DNN parameters, but they are also successful in black-box attacks, i.e., it has no access to the parameters. Black-box attacks are successful because one can perturb an image so it misclassifies on one DNN, and the same perturbed image also has a significant chance to be misclassified by another DNN; this is known as transferability of adversarial examples \cite{DBLP:journals/corr/PapernotMG16}. Due to this transferability, it is very easy to attack neural nets in a black-box fashion \cite{Brendel:2017}. In fact, there exist universal perturbations that can imperceptibly perturb any image and cause misclassification for any given network \cite{Moosavi-Dezfooli_2017_CVPR}. There is much recent research on designing advanced adversarial attacks and defending against adversarial perturbation. 

Defensive distillation was recently proposed to increase the stability of DNN \cite{PapernotDistillationDefense:2016}, and a related approach \cite{tramer2018ensemble} cleverly modifies the training data to increase robustness against black-box attacks, and adversarial attacks in general. To counter the adversarial perturbations, Guo et al. \cite{ChuanGuo:2018}, proposed to use image transformation, e.g., bit-depth reduction, JPEG compression, TVM, and image quilting. Adversarial training is another family of defense methods to improve the stability of DNN \cite{Goodfellow:2014AdversarialTraining}. In particular, the projected gradient descent (PGD) adversarial training achieves state-of-the-art guaranteed resistance to the first order attack \cite{Madry:2018}. Generative adversarial nets (GANs) are also employed for adversarial defense \cite{Samangouei:2018}. In \cite{Athalye:2018}, the authors proposed an approximated gradient to attack the defenses that is based on the obfuscated gradient. Wang et al. \cite{BaoWang:2018NIPS,BaoWang:2018Defense}, introduce a data dependent activation to defend against adversarial attacks, joint with total variation minimization, training data augmentation, and the PGD adversarial training, state-of-the-art adversarial defense results are achieved. More recently, motivated by the Feynman-Kac formalism, Wang et al. \cite{BaoWang:2018EnResNet}, proposed a novel neural nets ensemble algorithm which significantly improves the guaranteed robustness towards the first order adversarial attack.

%Oberman et al. \cite{Oberman:2018}, applied the Lipschitz regularization to improve the robustness of deep neural nets.

In this paper, we analyze the efficacy of the fast gradient sign method (FGSM) \cite{Goodfellow:2014AdversarialTraining,Kurakin:2016} and the Carlini-Wagner's L$_2$ (CW-L2) attack \cite{CWAttack:2016}. FGSM belongs to the fixed perturbation attack, while CW-L2 attack belongs to the zero-confidence attack. For FGSM, we consider both the targeted and the untargeted attacks. We prove that, within a specific regime, the untargeted FGSM can fool any convolutional neural nets (CNN) with ReLU activation; the targeted FGSM can mislead any CNN with ReLU activation to classify any given image into any prescribed class. For a two-layer neural nets, a linear layer followed by the softmax output activation, we show that the CW-L2 attack increases the ratio of the classification probability between the target and ground truth classes. Our theoretical results give guidance on applying different attacks to attack neural nets, especially, the targeted ones.

This paper is structured in the following way: In section \ref{Review-of-Attack}, we give a review of the well known adversarial attack schemes and briefly discuss the mathematical principle behind them. We analyze the untargeted FGSM, the targeted FGSM, and the CW-L2 attacks, respectively, in sections \ref{Untargeted-FGSM}, \ref{Target-FGSM}, \ref{Target-CWL2}. We verify our theoretical results numerically in section \ref{Numerical-Results}. The paper ends up with concluding remarks.

\section{Adversarial Attacks} \label{Review-of-Attack}
We denote the classifier defined by the DNN with softmax output activation as $\tilde{y} = f(\theta, x)$ for a given image-label pair ($x$, $y$). FGSM finds the adversarial image $x'$ by maximizing the loss $L(x', y) \doteq L(f(\theta, x'), y)$, subject to the $l_\infty$ perturbation constraint $||x'-x||_\infty \leq \epsilon$ with $\epsilon$ be the attack strength. Under the first order approximation i.e., $L(x', y) \approx L(x, y) + \nabla_xL(x, y)^T \cdot (x'-x)$, we have
\begin{equation}
\label{FGSM}
x'=x + \epsilon \, {\rm sign} \cdot \left( \nabla_xL(x, y) \right).
\end{equation}

IFGSM iterates FGSM to generate enhanced attacks, i.e.,
\begin{equation}
\label{IFGSM}
x^{(m)} = x^{(m-1)} + \epsilon \cdot {\rm sign} \left( \nabla_{x} L(x^{(m-1)}, y) \right),
\end{equation}
where $m=1, \cdots, M$, $x^{(0)}=x$ and $x'=x^{(M)}$, with $M$ being the number of iterations.

In practice, we apply the following clipped IFGSM
\begin{equation}
\label{IFGSM-2}
x^{(m)} = {\rm Clip}_{x, \alpha}\left\{x^{(m-1)} + \epsilon \cdot {\rm sign} \left( \nabla_x L(x^{(m-1)}, y) \right)\right\},
\end{equation}
where $\alpha$ is an additional parameter to be specified in the experiments.

\begin{remark}
The above FGSM or IFGSM attack fools DNN to mis-classify the image $x$. To mislead the classification result falls into any given class $t$, with one-hot label $e_t$, we apply the following targeted FGSM schemes
%neural nets to classify $x$ to any given class $t$, we can modify them to
\begin{itemize}
    \item Targted FGSM
    \begin{equation}
    \label{Targeted-FGSM}
    x'=x - \epsilon \, {\rm sign} \cdot \left( \nabla_xL(x, e_t) \right).
    \end{equation}
    
    \item Targted IFGSM
    \begin{equation}
    \label{Targeted-IFGSM}
    x^{(m)} = x^{(m-1)} - \epsilon \cdot {\rm sign} \left( \nabla_{x} L(x^{(m-1)}, e_t) \right),
    \end{equation}
    where $m=1, \cdots, M$, $x^{(0)}=x$ and $x'=x^{(M)}$, with $M$ being the number of iterations.
\end{itemize}
In Eqs.~(\ref{Targeted-FGSM}, \ref{Targeted-IFGSM}), $L(x, e_t)$ is the loss between predicted label of the adversarial image and the targeted label $e_t$.
\end{remark}

Furthermore, we consider the following zero-confidence attack. For a given image-label pair $(x, y)$, and $\forall t\neq y$, CW-L2 searches the adversarial image that will be classified to class $t$ by solving the optimization problem:
\begin{equation}
\label{cwl2-eq1}
\min_{\delta} ||\delta||_2^2,\ \ \ \ \ \text{subject to}\ f(x+\delta) = t, \; x+\delta \in [0, 1]^n,
\end{equation}
where $\delta$ is the adversarial perturbation (for simplicity, we ignore the dependence of $\theta$ in $f$).

The equality constraint in Eq.~(\ref{cwl2-eq1}) is hard to handle, so Carlini et al. consider the surrogate
\begin{equation}
\label{cwl2-eq2}
g(x) = \max\left(\max_{i\neq t}(Z(x)_i) - Z(x)_t , 0\right),
\end{equation}
where $Z(x)$ is the logit vector for an input $x$, i.e., output of the neural nets before the softmax layer. $Z(x)_i$ is the logit value corresponding to class $i$. It is easy to see that $f(x+\delta)=t$ is equivalent to $g(x+\delta)\leq 0$. Therefore, the problem in Eq.~(\ref{cwl2-eq1}) can be reformulated as
\begin{equation}
\label{cwl2-eq4}
\min_{\delta} ||\delta||_2^2 + c \cdot g(x+\delta)\ \ \ \ \ \text{subject to}\ x+\delta \in [0, 1]^n,
\end{equation}
where $c\geq 0$ is the Lagrangian multiplier.

By letting $\delta = \frac{1}{2}\left(\tanh(w)+1\right)-x$, Eq.~(\ref{cwl2-eq4}) can be written as an unconstrained problem. Moreover, Carlini et al. introduce the confidence parameter $\kappa$ into the above formulation. Above all, CW-L2 attacks seek the adversarial image by solving the following problem
\begin{align}
\label{CWL2}
&\min_{w} ||\frac{1}{2}\left(\tanh(w) + 1\right) - x ||_2^2 + c\cdot\\ \nonumber
&\max\left\{-\kappa, \max_{i\neq t}(Z(\frac{1}{2}(\tanh(w))+1)_i)
- Z(\frac{1}{2}(\tanh(w))+1)_t \right\}.
\end{align}

This unconstrained problem can be solved efficiently by the Adam optimizer \cite{Kingma:2014Adam}. 
%All three of the attacks clip the values of the adversarial image $x'$ to between 0 and 1.

In the case of CW-L2 attack, we introduce different levels of adversarial attack by setting the adversarial image to
{\small
\begin{equation}
\label{CWL2-eps}
x' = x + \epsilon\left(x^{\rm adv} - x\right),
\end{equation}
}
where $x^{\rm adv}$ is the solution to Eq.~(\ref{CWL2}).

\begin{figure}[h]
\centering
\begin{tabular}{cc}
\includegraphics[width=0.45\columnwidth]{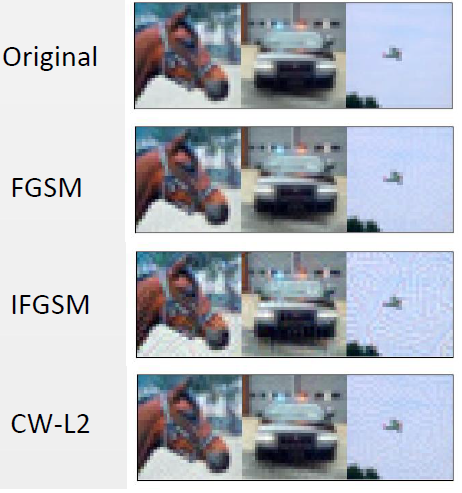}&
\includegraphics[width=0.45\columnwidth]{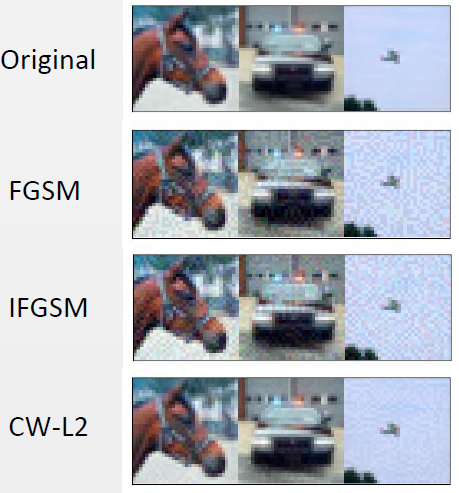}\\
(a)&(b)\\
\end{tabular}
\caption{Samples from CIFAR10. Panel (a): from the top to the last rows show the original, adversarial images by attacking ResNet56 with FGSM, IFGSM, CW-L2 ($\epsilon=0.02$). Panel (b) corresponding to those in panel (a) with $\epsilon=0.08$.}
\label{fig:Adv-images}
\end{figure}

Figure~\ref{fig:Adv-images} depicts three randomly selected images (horse, automobile, airplane) from the CIFAR10 dataset, their adversarials by using different attacks on ResNet56. All attacks fool the classifiers completely on these images. Figure~\ref{fig:Adv-images} (a) shows that the perturbations resulted from FGSM and IFGSM with $\epsilon=0.02$, 10 iterations with $\alpha=0.1$ for IFGSM, is almost imperceptible. For CW-L2, we set the parameters $c=10$ and $\kappa=0$, and run 10 iterations of Adam optimizer with learning rate 0.01. Figure~\ref{fig:Adv-images} (b) shows the corresponding images of (a) with a stronger attack, $\epsilon=0.08$. With a larger $\epsilon$, the adversarial images become more noisy. 

\section{Analysis of the Untargeted FGSM}\label{Untargeted-FGSM}
\subsection{Case 1. A linear layer followed by a softmax output layer}
For an input image $x\in R^{n}$ and its corresponding one-hot label vector $y\in R^{k}$. We consider the simple neural nets
\begin{equation}\label{NeuralNets1-1}
\hat{y}={\rm softmax}(Wx),
\end{equation}
and consider the cross entropy loss
\[L(x, y) = {\rm crossentropy}(\hat{y}, y) = -\sum_{j=1}^{k}y_{j}\cdot\ln\hat{y}_{j} = -\ln{\hat{y}_{s}},\]
where $W\in R^{k\times n}$, $s$ is the index of 1 in the one-hot vector $y$ i.e., $y_{s} = 1$ and $y_{i} = 0~~\forall i\neq s$.
\begin{theorem} \label{Thm-1}
For the neural nets defined in Eq.~(\ref{NeuralNets1-1}) and any input-output pair $(x, y)$. Let $x'$ be the adversarial image of $x$ resulting from FGSM attack, i.e., 
%\[x' = x + \epsilon\cdot sign(\nabla_{x}L(x, y, W)).\]
\[x' = x + \epsilon\cdot sign(\nabla_{x}L(x, y)).\]
Then, for $\forall \epsilon > 0$, we have:
\[L(x, y) \leqslant L(x', y). \]
%That is to say, after applying FGSM to attack the original input $x$ to $x'$, the loss function will %surely increase, no matter what $x,y,W$ and $\epsilon > 0$ are.
\end{theorem}
\begin{proof}
%We randomly pick an input $x$ and its corresponding one-hot label vector $y$. 
For any given $x$, suppose it belongs to class $s$.
%assume $y_{s} = 1$ and all of the other components of $y$ are 0.
The loss can be expressed as:

\begin{equation}\label{2}
\begin{aligned}
L(x, y) &= {\rm crossentropy}({\rm softmax}(Wx), y)\\
&= -{\rm ln}({\rm softmax}(Wx)_{s})\\
&= -{\rm ln}(\frac{{\rm exp}(Wx)_{s}}{{\rm exp}(Wx)_{1}+{\rm exp}(Wx)_{2}+\cdots+{\rm exp}(Wx)_{k}})
\end{aligned}
\end{equation}

Here,
$$
{\rm softmax}(Wx)_{i} = \frac{{\rm exp}(Wx)_{i}}{\sum_{j=1}^{k}{\rm exp}(Wx)_{j}}.
$$
%After calculating the partial derivatives of our loss function above, we can get the accurate expression of $x'$.
It is easy to get the exact expression of $x'$, in fact,
\begin{equation}
\begin{aligned}
x'_{i} &= x_{i} + \epsilon\cdot {\rm sign}(\frac{\partial}{\partial x_{i}}L(x,y))\\
&= x_{i} + \epsilon\cdot \alpha_{i}.
\end{aligned}
\end{equation}
Here:

\begin{equation}\label{4}
\begin{aligned}
\alpha_{i} &= {\rm sign}(\frac{\partial}{\partial x_{i}}L(x,y))\\
&= {\rm sign}(-\frac{\sum_{j=1}^{k}{\rm exp}(Wx)_{j}}{{\rm exp}(Wx)_{s}}\cdot\\
&~~~\frac{{\rm exp}(Wx)_{s}w_{si}(\sum_{j=1}^{k}{\rm exp}(Wx)_{j})-{\rm exp}(Wx)_{s}\sum_{j=1}^{k}{\rm exp}(Wx)_{j}w_{ji}}{(\sum_{j=1}^{k}{\rm exp}(Wx)_{j})^{2}})\\
&= -{\rm sign}(w_{si}(\sum_{j=1}^{k}{\rm exp}(Wx)_{j})-\sum_{j=1}^{k}{\rm exp}(Wx)_{j}w_{ji})\\
&= {\rm sign}(\sum_{j=1}^{k}{\rm exp}(Wx)_{j}w_{ji}-(\sum_{j=1}^{k}{\rm exp}(Wx)_{j})w_{si}).
\end{aligned}
\end{equation}

%Now, it is time to make a comparison between $L(x,y)$ and $L(x',y)$. 
In order to prove $L(x, y)\leqslant L(x', y)$, we only need to show

\begin{equation}\label{5}
\begin{aligned}
&\frac{{\rm exp}(Wx)_{s}}{\sum_{j=1}^{k}{\rm exp}(Wx)_{j}}\geqslant \frac{{\rm exp}(Wx')_{s}}{\sum_{j=1}^{k}{\rm exp}(Wx')_{j}}\\
\Leftrightarrow & \frac{{\rm exp}(Wx')_{s}}{{\rm exp}(Wx)_{s}}\leqslant \frac{{\rm exp}(Wx')_{1} + {\rm exp}(Wx')_{2} + \cdots + {\rm exp}(Wx')_{k}}{\sum_{j=1}^{k}{\rm exp}(Wx)_{j}}\\
\Leftrightarrow & \frac{{\rm exp}(Wx')_{s}}{{\rm exp}(Wx)_{s}}\leqslant \sum_{j=1}^{k}{\rm softmax}(Wx)_{j}\cdot\frac{{\rm exp}(Wx')_{j}}{{\rm exp}(Wx)_{j}}\\
\Leftrightarrow & {\rm exp}(\epsilon W\alpha)_{s}\leqslant \sum_{j=1}^{k}{\rm softmax}(Wx)_{j}\cdot {\rm exp}(\epsilon W\alpha)_{j}.
\end{aligned}
\end{equation}

Here, $\alpha=(\alpha_{1}, \alpha_{2}, \cdots, \alpha_{k})^{T}$. Since ${\rm softmax}(Wx)_{i}, 1\leqslant i \leqslant k$, are $k$ non-negative real numbers sum to 1. By the Jensen's inequality, we can get the following lower bound for the right hand side of Eq.~(\ref{5})

%Now in order to prove the inequality \ref{5}, let's introduce the famous Jensen's Inequality. Since the %exponent function is a convex function, this inequality is legal to be used.
%\begin{lemma}\label{Lemma-2}\textbf{The Jensen's Inequality}\\
%Assume $a_{1},a_{2},\cdots,a_{n}$ are $n$ non-negative real numbers and they satisfy the condition %$a_{1}+a_{2}+\cdots+a_{n}=1$. There are also $n$ real numbers $x_{1},x_{2},\cdots,x_{n}$, then the following %inequality holds:
%\[a_{1}exp(x_{1})+a_{2}exp(x_{2})+\cdots+a_{n}exp(x_{n})\geqslant %exp(a_{1}x_{1}+a_{2}x_{2}+\cdots+a_{n}x_{n})\]
%\end{lemma}

%So, according to the lemma above, along with the natural condition of:
%\[softmax(Wx)_{i}~~~(1\leqslant i \leqslant k)\] 
%are $k$ non-negative real numbers and there sum is 1, we can get a lower bound of the right hand side of %\ref{5}.

\begin{equation}
\begin{aligned}
& {\rm softmax}(Wx)_{1}\cdot {\rm exp}(\epsilon W\alpha)_{1}+\cdots + {\rm softmax}(Wx)_{k}\cdot {\rm exp}(\epsilon W\alpha)_{k}\\
\geqslant & {\rm exp}(\sum_{j=1}^{k}\epsilon\cdot {\rm softmax}(Wx)_{j}(W\alpha)_{j}).
\end{aligned}
\end{equation}

So far, in order to get Eq.~(\ref{5}) proved, we only have to prove the following inequality% a much simpler and $\epsilon$-free inequality:
\begin{equation}
\sum_{j=1}^{k} {\rm softmax}(Wx)_{j}(W\alpha)_{j}\geqslant (W\alpha)_{s}.
\end{equation}
This is equivalent to :
\begin{equation}
\sum_{j=1}^{k}{\rm exp}(Wx)_{j}(W\alpha)_{j}\geqslant ({\rm exp}(Wx)_{1}+\cdots+{\rm exp}(Wx)_{k})(W\alpha)_{s}.
\end{equation}
In fact, we have

\begin{equation}
\begin{aligned}
& \sum_{j=1}^{k}{\rm exp}(Wx)_{j}(W\alpha)_{j} - ({\rm exp}(Wx)_{1} + \cdots + {\rm exp}(Wx)_{k})(W\alpha)_{s}\\
=& \sum_{j=1}^{k}{\rm exp}(Wx)_{j}(w_{j1}\alpha_{1} + \cdots + w_{jn}\alpha_{n}) - (\sum_{j=1}^{k}{\rm exp}(Wx)_{j})(w_{s1}\alpha_{1} + \cdots + w_{sn}\alpha_{n})\\
=& \sum_{t=1}^{n}\alpha_{t}(\sum_{j=1}^{k}{\rm exp}(Wx)_{j}w_{jt} - (\sum_{j=1}^{k}{\rm exp}(Wx)_{j})w_{st})\\
\geqslant &0.
\end{aligned}
\end{equation}

The last step uses Eq.~(\ref{4}) and the fact that, $x\cdot sign(x)\geqslant 0$.
Now the theorem is proved.
\end{proof}

\subsection{Case 2. Two linear layers followed by softmax output layers, with ReLU Activation in the middle}
%slightly more complex model, i.e., two linear and softmax output layers, with ReLU activation
\begin{equation}\label{10}
\begin{aligned}
\hat{y} = {\rm softmax}(V\sigma(Wx)),
\end{aligned}
\end{equation}
again, we consider the cross entropy loss:
\[L(x, y) = {\rm crossentropy}(\hat{y}, y) = -\sum_{j=1}^{k}y_{j}\cdot\ln\hat{y}_{j} = -\ln{\hat{y}_{s}},\]
where $W\in R^{l\times n}, V\in R^{k\times l}, x\in R^{n}, y\in R^{k}$. $\sigma$ is the ReLU activation.

\begin{theorem} \label{Thm-3.2}
For the neural nets defined in Eq.~(\ref{10}) and any input-output pair $(x, y)$. Let $x'$ be the adversarial of $x$ by applying FGSM attack to Eq.~(\ref{10}), i.e.,
\[x' = x + \epsilon\cdot {\rm sign}(\nabla_{x}L(x, y)).\]
Suppose every element of $Wx$ is non-zero, if $\epsilon < \frac{|Wx|_{min}}{\|W\|_{\infty}}$, then we have:
\[L(x, y) \leqslant L(x', y). \]
Here, $|Wx|_{min}$ is the smallest element among the absolute values of $Wx$, $\|W\|_{\infty}$ is the infinity norm of matrix $W$, i.e.,
$$
||W||_{\infty} = \max_i(|w_{i1}|+|w_{i2}|+\cdots+|w_{in}|).
$$
%In other words, as long as $\epsilon$ is sufficiently small, after using FGSM attack to replace the original %input $x$ with $x'$, the loss function will surely increase, no matter what $x,y,W$ are.
\end{theorem}

\begin{proof}
Let $T \doteq V\sigma W$, and $\hat{y} = {\rm softmax}(Tx)$. First, we introduce a simple lemma.
\begin{lemma}\label{lemma3-3}
For $j = 1, 2, \cdots, l$, ${\rm sign}(Wx)_{j} = sign(Wx')_{j}$.
\end{lemma}
\begin{proof}
Let $x' = x + \delta$, every element of $\delta$ is one of $\epsilon, -\epsilon$, $0$. Since
\[(Wx')_{j} = (Wx)_{j} +(W\delta)_{j},\]
and :
\begin{equation}
\begin{aligned}
|(W\delta)_{j}| &= |\sum_{i=1}^{n}w_{ji}\delta_{i}|\leqslant \sum_{i=1}^{n}|w_{ji}|\cdot|\delta_{i}|\\
&\leqslant \sum_{i=1}^{n}\epsilon |w_{ji}| = \epsilon\cdot \sum_{i=1}^{n}|w_{ji}|\\
&\leqslant \epsilon \|W\|_{\infty} < |Wx|_{min} \leqslant |(Wx)_{j}|.\\
\end{aligned}
\end{equation}
Therefore: $(Wx)_{j}$ and $(Wx')_{j}$ have the same sign.
\end{proof}

We go back to proof the Theorem.~(\ref{Thm-3.2}). Let us define the following index set
\[A \doteq \{i:(Wx)_{i} > 0\} = \{i:(Wx')_{i} > 0\}.\]
Then we can express the operator $T$ as:
\begin{equation}
\begin{aligned}
(Tx)_{j} &= \sum_{t=1}^{l}v_{jt}\sigma(w_{t1}x_{1}+w_{t2}x_{2}+\cdots+w_{tn}x_{n})\\
&= \sum_{t\in A}v_{jt}(w_{t1}x_{1}+w_{t2}x_{2}+\cdots+w_{tn}x_{n}),
\end{aligned}
\end{equation}
So, the operator $T$ is a linear operator.

%Then we calculate our loss function and the partial derivatives. 
The loss function can be rewritten as
\begin{equation}
\begin{aligned}
L(x, y) &= {\rm crossentropy}({\rm softmax}(Tx),y)\\
&= -\ln({\rm softmax}(Tx)_{s})\\
&= -\ln(\frac{{\rm exp}(Tx)_{s}}{{\rm exp}(Tx)_{1} + {\rm exp}(Tx)_{2}+\cdots + {\rm exp}(Tx)_{k}}).
\end{aligned}
\end{equation}
Therefore, this case can be reduced to Case 1, where our model consists of a linear layer followed by softmax output layer. Similar to the previous case, every element of the adversarial image can be written as:
\begin{equation}
\begin{aligned}
x'_{i} &= x_{i} + \epsilon\cdot {\rm sign}(\frac{\partial}{\partial x_{i}}L(x, y))\\
&= x_{i} + \epsilon\cdot \alpha_{i},
\end{aligned}
\end{equation}
Replacing $w_{ji}$ with $\sum_{a\in A}v_{ja}w_{ai}$ in Equation \ref{4}:
\begin{equation}\label{15}
\begin{aligned}
\alpha_{i} &= {\rm sign}(\frac{\partial}{\partial x_{i}}L(x,y))\\
&= {\rm sign}(\sum_{j=1}^{k}{\rm exp}(Tx)_{j}(\sum_{a\in A}v_{ja}w_{ai})-(\sum_{j=1}^{k}{\rm exp}(Tx)_{j})\sum_{a\in A}v_{sa}w_{ai}).
\end{aligned}
\end{equation}
%Now, we make a comparison between $L(x,y)$ and $L(x',y)$. In order to prove 
Completely similar to Case 1, after using Jensen's Inequality, we can get this theorem proved. \\

Remark: The usage of the upper bound of $\epsilon$ is to make the operator $T = V\sigma W$ a locally linear transformation. So that the Case 2 is reduced to Case 1 which is previously proved.
%According to Equation \ref{15}, the inequality has been proved.\\
%Till now, we have proved this theorem.
\end{proof}

\subsection{Case 3. Multiple linear and softmax output layers, with all ReLU activations}
From Case 2, we note that when $\epsilon$ is small enough, everything inside the ReLU activation does not change the sign after the FGSM attack. Under this condition, the part before the softmax output layer can be treated as a linear function. Therefore, according to our proof in Case 1 where the neural nets consists of only a single linear layer before the softmax, the theorem in Case 2 is undoubtedly correct.
Therefore, we can generalize Case 2 to the neural nets consists of multiple linear and a softmax output layers, with all activations between linear layers set be ReLU.

Denote our training neural network as:
\begin{equation}\label{20}
\hat{y} = softmax(W_{L}\sigma W_{L-1}\sigma\cdots\sigma(W_{1}x))
\end{equation}

and we denote:
\[Tmp_{1}=W_{1}x, ~~Tmp_{2}=W_{2}\sigma(Tmp_{1}),~\cdots~, ~~Tmp_{L} = W_{L}\sigma(Tmp_{L-1})\]

Here, $L$ is the layer number of our neural network. $W_{i}~~(1\leqslant i\leqslant L)$ are matrices and $\sigma$ stands for the ReLU activation function. Besides, $Tmp_{i}~~(1\leqslant i\leqslant L)$ are the intermediate results in the neural network.\\

Therefore we give our general theorem.
\begin{theorem}
For the training neural network \ref{20} and any input and output $x,y$. Let $x'$ is the attacking result of the original input $x$ with FGSM:
\[x' = x + \epsilon\cdot sign(\nabla_{x}L(x,y,W))\]
Then, assume every element of $Tmp_{i}~(1\leqslant i\leqslant L)$ is non-zero and $\epsilon$ is sufficiently small so that every element in all $Tmp_{i}$ do not change their sign after the $\epsilon$-FGSM attack, then we have:
\[L(x,y) \leqslant L(x',y) \]
In other words, as long as $\epsilon$ is sufficiently small, after using FGSM attack to replace the original input $x$ with $x'$, the loss function will surely increase, no matter what $x,y,W_{i}~(1\leqslant i\leqslant L)$ are.
\end{theorem}

The proof of the theorem is similar to the one in Case 2. Since every element in all $Tmp_{i}$ do not change their sign, therefore $W_{L}\sigma\cdots\sigma(W_{1}x)$ can be seen as a linear function during this attack. In this condition, this problem is equivalent to Case 1. 

Finally, we give an upper bound of $\epsilon$ to satisfy the condition of the theorem above.
\[\epsilon < \min_{1\leqslant j\leqslant L}\frac{|Tmp_{j}|_{min}}{\|W_{j}\|_{\infty}}\]

\subsection{Remark on the convolutional layer}\label{sec3-4}
In all the cases above, there is an assumption that all the layers are linear. We can also generalize this to convolutional layers. Note for a convolutional layer $h$ and an input matrix $X$, when we flatten $X$ and the result $h(X)$ to a column vector, $h$ is also linear. 
%Since the softmax output layer is also a column vector, 
Therefore the convolutional layers can be regarded as a linear layer when we flatten all the input and intermediate matrices.

Overall, when the neural nets consists of linear or convolutional layers, with a softmax output layer and ReLU activation, then the efficacy of FGSM attack can be guaranteed as long as the $\epsilon$ is sufficiently small.

\section{Targeted Fast Gradient Sign Method} \label{Target-FGSM}
In this section, we consider the efficacy of the targeted adversarial attack with FGSM.
Given any input $x\in R^{n}$ and its corresponding one-hot label vector $y\in R^{k}$, we want to attack it so that the new output falls into the $t$-th category. Considering the following targeted FGSM
\begin{equation}
%x' = x - \epsilon\cdot(\nabla_{x}L(x, e_{t}, W)),
x' = x - \epsilon\cdot (\nabla_x L(x, e_t)),
\end{equation}
where $e_t$ is the one-hot vector of class $t$.
%where, the $e_{t}$ is a column vector with its $t$-th entry be 1 and all the others be 0. We turn the plus %into minus and turn the original $y$ into $e_{t}$, so that the output of $x'$ can be closer to $e_{t}$. \\

%Next, we are going to give the similar theoretical guarantee of our targeted FGSM.

%\subsection{Case 1. One linear and a softmax output layers}
\subsection{Case 1. A linear layer followed by a softmax output layer}
%For any input $x\in R^{n}$ and its corresponding one-hot label vector $y\in R^{k}$, and the linear layer $W\in R^{k\times n}$, our training neural network can be expressed as:
Again, we first consider a very simple neural nets,
\begin{equation}\label{5-2}
\hat{y}={\rm softmax}(Wx)
\end{equation}
%and our loss function is:
with cross-entropy loss
\[L(x,y) = {\rm crossentropy}(\hat{y}, e_{t}) = -\sum_{j=1}^{k}y_{j}\cdot\ln(e_{t})_{j} = -\ln{\hat{y}_{t}}\]

\begin{theorem}
For the neural nets define by Eq.~(\ref{5-2}), any input-output pair $(x, y)$, and the target label $t$. Let $x'$ be the adversarial of $x$ resulting from the targeted FGSM attack, i.e.,
%is the attacking result of the original input $x$ with the targeted FGSM,
\[x' = x - \epsilon\cdot {\rm sign}(\nabla_{x}L(x, e_{t})).\]
Under the assumption that $\nabla_{x}L(x, e_{t})$ %the derivative above 
has no zero elements, for any:
\[\epsilon < \min_{i}\frac{1}{\|W\|_{\infty}}\ln{(1+\frac{|\sum_{j=1}^{k}\exp(Wx)_{j}w_{ji}-w_{ti}(\sum_{j=1}^{k}\exp(Wx)_{j})|}{\sum_{j=1}^{k}|w_{ji}-w_{ti}|\cdot \exp(Wx)_{j}})}\]
we have:
\[L(x, e_{t}) \geqslant L(x', e_{t}). \]
%That is to say, after using targeted FGSM attack to replace the original input $x$ with $x'$, the loss function %will surely decrease, no matter what $x,y,W$ and $\epsilon > 0$ are.
\end{theorem}

\begin{proof}
%We randomly pick an input $x$ and its corresponding one-hot label vector $y$. 
%Since the loss function can be expressed as:
The loss associated with the target label $t$ for any given image-label pair $(x, y)$ is
\begin{equation}
\begin{aligned}
L(x, e_t) &= {\rm crossentropy}({\rm softmax}(Wx), e_{t})\\
&= -\ln({\rm softmax}(Wx)_{t})\\
&= -\ln(\frac{\exp(Wx)_{t}}{\exp(Wx)_{1} + \exp(Wx)_{2}+\cdots+\exp(Wx)_{k}}),
\end{aligned}
\end{equation}
where, $e_t$ is the one hot vector for $t$-th class.

%Here,
%\[softmax(Wx)_{i} = %\frac{exp(Wx)_{i}}{\sum_{j=1}^{k}exp(Wx)_{j}}\]
%After calculating the partial derivatives of our loss %function above, we can get the accurate expression of %$x'$.\\

In fact,
\begin{equation}
\begin{aligned}
x'_{i} &= x_{i} - \epsilon\cdot {\rm sign}(\frac{\partial}{\partial x_{i}}L(x,e_{t}))\\
&= x_{i} - \epsilon\cdot \alpha_{i},
\end{aligned}
\end{equation}
where,
\begin{equation}
\begin{aligned}
\alpha_{i} &= {\rm sign}(\frac{\partial}{\partial x_{i}}L(x,e_{t}))\\
&= {\rm sign}(-\frac{\sum_{j=1}^{k}\exp(Wx)_{j}}{\exp(Wx)_{t}}\cdot\\
&~~~\frac{\exp(Wx)_{t}w_{ti}(\sum_{j=1}^{k}\exp(Wx)_{j})-\exp(Wx)_{t}\sum_{j=1}^{k}\exp(Wx)_{j}w_{ji}}{(\sum_{j=1}^{k}\exp(Wx)_{j})^{2}})\\
&= -{\rm sign}(w_{ti}(\sum_{j=1}^{k}\exp(Wx)_{j})-\sum_{j=1}^{k}\exp(Wx)_{j}w_{ji})\\
&= {\rm sign}(\sum_{j=1}^{k}\exp(Wx)_{j}w_{ji}-w_{ti}(\sum_{j=1}^{k}\exp(Wx)_{j})).
\end{aligned}
\end{equation}
According to the assumption, the derivative above is nonzero. Therefore, if $\epsilon$ is sufficiently small, we have
\begin{equation}\label{5-6}
\begin{aligned}
& {\rm sign}(\sum_{j=1}^{k}\exp(Wx)_{j}w_{ji}-w_{ti}(\sum_{j=1}^{k}\exp(Wx)_{j}))\\
=& {\rm sign}(\sum_{j=1}^{k}\exp(Wx')_{j}w_{ji}-w_{ti}(\sum_{j=1}^{k}\exp(Wx')_{j})).
\end{aligned}
\end{equation}
Then, under this condition:
\begin{equation}\label{5-7}
\alpha_{i} = {\rm sign}(\sum_{j=1}^{k}\exp(Wx')_{j}w_{ji}-w_{ti}(\sum_{j=1}^{k}\exp(Wx')_{j})).
\end{equation}

We would like to give a upper bound of $\epsilon$ such that Eq.~(\ref{5-6}) holds. Actually, since every entry of $x$ and $x'$ has a difference at most $\epsilon$, then $|(Wx)_{j}-(Wx')_{j}|\leqslant \|W\|_{\infty}\epsilon$. Therefore:
\[|\exp(Wx)_{j}-\exp(Wx')_{j}|\leqslant \exp(Wx)_{j}(\exp(\|W\|_{\infty}\epsilon)-1).\]
Let
\[A = \sum_{j=1}^{k}\exp(Wx)_{j}w_{ji}-w_{ti}(\sum_{j=1}^{k}\exp(Wx)_{j}),\]
\[B = \sum_{j=1}^{k}\exp(Wx')_{j}w_{ji}-w_{ti}(\sum_{j=1}^{k}\exp(Wx')_{j}).\]
So after the attack, the difference
\begin{equation}
\begin{aligned}
|A-B|&\leqslant\sum_{j=1}^{k}|w_{ji}-w_{ti}|\cdot|\exp(Wx)_{j} - \exp(Wx')_{j}|\\
&\leqslant\sum_{j=1}^{k}|w_{ji}-w_{ti}|\cdot \exp(Wx)_{j}(\exp(\|W\|_{\infty}\epsilon)-1)\\
&= (\exp(\|W\|_{\infty}\epsilon)-1)\sum_{j=1}^{k}|w_{ji}-w_{ti}|\cdot \exp(Wx)_{j}.
\end{aligned}
\end{equation}
When $\epsilon$ satisfies
\begin{equation}
\begin{aligned}
\epsilon &< \frac{1}{\|W\|_{\infty}}\ln{(1+\frac{|A|}{\sum_{j=1}^{k}|w_{ji}-w_{ti}|\cdot \exp(Wx)_{j}})}\\
&= \frac{1}{\|W\|_{\infty}}\ln{(1+\frac{|\sum_{j=1}^{k}\exp(Wx)_{j}w_{ji}-w_{ti}(\sum_{j=1}^{k}\exp(Wx)_{j})|}{\sum_{j=1}^{k}|w_{ji}-w_{ti}|\cdot \exp(Wx)_{j}})}.
\end{aligned}
\end{equation}
In summary, when
\[\epsilon < \min_{i}\frac{1}{\|W\|_{\infty}}\ln{(1+\frac{|\sum_{j=1}^{k}\exp(Wx)_{j}w_{ji}-w_{ti}(\sum_{j=1}^{k}\exp(Wx)_{j})|}{\sum_{j=1}^{k}|w_{ji}-w_{ti}|\cdot \exp(Wx)_{j}})}\]
Equation \ref{5-7} holds.
%Now, it is time to make a comparison between $L(x,e_{t})$ and %$L(x',e_{t})$. In order to get $L(x,e_{t})\geqslant L(x',e_{t})$, we %only have to prove:

Now, we show that $L(x, e_t)\geqslant L(x', e_t)$
\begin{equation}\label{5-8}
\begin{aligned}
&\frac{\exp(Wx)_{t}}{\sum_{j=1}^{k}\exp(Wx)_{j}}\leqslant \frac{\exp(Wx')_{t}}{\sum_{j=1}^{k}\exp(Wx')_{j}}\\
\Leftrightarrow & \frac{\exp(Wx)_{t}}{\exp(Wx')_{t}}\leqslant \frac{\exp(Wx)_{1} + \exp(Wx)_{2} + \cdots + \exp(Wx)_{k}}{\sum_{j=1}^{k}\exp(Wx')_{j}}\\
\Leftrightarrow & \frac{\exp(Wx)_{t}}{\exp(Wx')_{t}}\leqslant \sum_{j=1}^{k}{\rm softmax}(Wx')_{j}\cdot\frac{\exp(Wx)_{j}}{\exp(Wx')_{j}}\\
\Leftrightarrow & \exp(\epsilon W\alpha)_{t}\leqslant \sum_{j=1}^{k}{\rm softmax}(Wx')_{j}\cdot \exp(\epsilon W\alpha)_{j}.
\end{aligned}
\end{equation}

This is exactly Equation \ref{5}, which has been previously proved.
%Then, the theorem is proved.
\end{proof}

\subsection{Case 2. Two linear and softmax output layers, with ReLU activation}
In this part, we consider the neural nets with two linear and softmax output layers, with ReLU activation
\begin{equation}\label{5-15}
\begin{aligned}
\hat{y} = {\rm softmax}(V\sigma(Wx)),
\end{aligned}
\end{equation}
the loss is
\[L(x, e_t) = {\rm crossentropy}(\hat{y}, e_t) = -\sum_{j=1}^{k}(e_{t})_{j}\cdot\ln\hat{y}_{j} = -\ln{\hat{y}_{t}}.\]
Here: $W\in R^{l\times n}, V\in R^{k\times l}, x\in R^{n}, y\in R^{k}$. $\sigma$ is the ReLU activation function.

\begin{theorem}\label{Thm-10}
For the neural net defined in Eq.~(\ref{5-15}) and any input-output pair $(x, y)$. Let $x'$ be the adversarial of $x$ resulting from targeted FGSM attack
\[x' = x - \epsilon\cdot sign(\nabla_{x}L(x, e_t)),\]
Suppose the derivative above has no zero elements, every element of $Wx$ is non-zero, and $\epsilon$ is smaller than an upper bound which will be written at the end of the proof below, we have
\[L(x, e_t) \geqslant L(x', e_t). \]
%In other words, as long as $\epsilon$ is sufficiently small, after %using FGSM attack to replace the original input $x$ with $x'$, the %loss function will surely decrease, no matter what $x,y,W$ are.
\end{theorem}

\begin{proof}
Let $T \doteq V\sigma W$, and $\hat{y} = {\rm softmax}(Tx)$.
According to Lemma \ref{lemma3-3},if:
\[\epsilon < \frac{|Wx|_{min}}{\|W\|_{\infty}},\]
then ${\rm sign}(Wx)_{j} = {\rm sign}(Wx')_{j}$. 

Denote the index set
\[A = \{i:(Wx)_{i} > 0\} = \{i:(Wx')_{i} > 0\}.\]
Then we can express the operator $T$ as:
\begin{equation}
\begin{aligned}
(Tx)_{j} &= \sum_{t=1}^{l}v_{jt}\sigma(w_{t1}x_{1}+w_{t2}x_{2}+\cdots+w_{tn}x_{n})\\
&= \sum_{t\in A}v_{jt}(w_{t1}x_{1}+w_{t2}x_{2}+\cdots+w_{tn}x_{n}),
\end{aligned}
\end{equation}
so, similar to Case 3.2, the operator $T$ can be regarded as a locally linear operator. And we only need to replace the $w_{ji}$ in Case 4.1 with $\sum_{a\in A}v_{ja}w_{ai}$. So once $\epsilon$ is controlled by the upper bound $U = min(U_{1}, U_{2})$, then the theorem is correct. Here:
\[U_{1} = \frac{|Wx|_{min}}{\|W\|_{\infty}}\]
\[U_{2} = \min_{i}\frac{1}{\|T\|_{\infty}}\ln{(1+\frac{|\sum_{j=1}^{k}\exp(Tx)_{j}t_{ji}-t_{ti}(\sum_{j=1}^{k}\exp(Tx)_{j})|}{\sum_{j=1}^{k}|t_{ji}-t_{ti}|\cdot \exp(Tx)_{j}})}\]
where:
\[t_{ji} = \sum_{a\in A}v_{ja}w_{ai}\]
\end{proof}

\subsection{Case 3. Multiple linear and softmax output layers, with all ReLU activations}
From Case 2, we note that when $\epsilon$ is small enough, we can guarantee that everything inside the ReLU activation does not change their sign after the $\epsilon$-FGSM attack. Under this condition, the part before the softmax can be treated as a linear function. Therefore, according to our proof in Case 1 where the neural nets consists of only a single linear layer before the softmax, the theorem in Case 2 remains correct. Therefore, we can generalize Case 2 to the neural nets consists of multiple linear and softmax output layers, with all activation functions between linear layers be ReLU.
Consider the neural nets
\begin{equation}\label{TargetedFGSMNets}
\hat{y} = {\rm softmax}(W_{L}\sigma W_{L-1}\sigma\cdots\sigma(W_{1}x))
\end{equation}

and we denote:
\[{\rm Tmp}_{1}\doteq W_{1}x, ~~{\rm Tmp}_{2} \doteq W_{2}\sigma({\rm Tmp}_{1}), ~\cdots~, ~~{\rm Tmp}_{L} = W_{L}\sigma({\rm Tmp}_{L-1})\]

Here, $L$ is the number of layers. $W_{i}~~(1\leqslant i\leqslant L)$ are matrices and $\sigma$ stands for the ReLU activation. Moreover, ${\rm Tmp}_{i}~~(1\leqslant i\leqslant L)$ are the intermediate results in the neural network.

For the neural nets defined in Eq.~(\ref{TargetedFGSMNets}), we give the following theorem
\begin{theorem}
For the neural net defined by Eq.~(\ref{TargetedFGSMNets}) and any input-output pair $(x, y)$. Let $x'$ be the adversarial of $x$ by the targeted FGSM, i.e.,
\[x' = x - \epsilon\cdot sign(\nabla_{x}L(x, e_{t})).\]
Suppose the derivative above has no zero elements, and every element of ${\rm Tmp}_{i}~(1\leqslant i\leqslant L)$ is non-zero and $\epsilon$ is sufficiently small so that every element in all ${\rm Tmp}_{i}$ do not change their sign after the $\epsilon$-targeted FGSM attack, then we have
\[L(x, e_t) \geqslant L(x', e_t) \]
%In other words, as long as $\epsilon$ is sufficiently small, after %using targeted FGSM attack to replace the original input $x$ with %$x'$, the loss function will surely decrease, no matter what %$x,y,W_{i}~(1\leqslant i\leqslant L)$ are.
\end{theorem}

The proof of the theorem is similar to the one in Case 2. Since every element in all ${\rm Tmp}_{i}$ does not change their sign, therefore $W_{L}\sigma\cdots\sigma(W_{1}x)$ can be seen as a linear function during this attack. In this condition, this problem is equivalent to Case 1. 

Remark: For the same reason as Section \ref{sec3-4}, this conclusion can be also generated to convolutional layers.

To sum up, when our training neural network consists of linear or convolution layers, with a softmax output layer and all acivations ReLU, then the efficacy of targeted FGSM adversarial attack can be guaranteed theoretically as long as the $\epsilon$ is sufficiently small.

\section{CW-L2 Targeted Adversarial Attack} \label{Target-CWL2}
\subsection{Models and attack}
In this section, we consider the simple neural nets which consists of a linear and softmax output layers, i.e., 
%we only consider the training model which consists of one linear layer and softmax output layer. In the other words, the training neural network is:
\begin{equation}
\hat{y}={\rm softmax}(Wx).
\end{equation}
We consider the simplified CW-L2 attack, in which we relax the constraint that pixel values are between 0 and 1, in this case, 
%The targeted adversarial attack method we use is the simplified CW-L2, where 
the perturbation $\delta$ is the solution of the optimization problem
\begin{equation}
\argmin_{\delta}%~L(x+\delta)
\|\delta\|_{2}^{2}+c\cdot g(x+\delta),
\end{equation}
where,
\[g(x)=\max\big(\max\limits_{i\neq t}(Z(x)_{i})-Z(x)_{t}, 0\big)\]
and $Z(x)$ is the logit vector of the input $x$, which in our one-layer network, means that:
\[Z(x) = Wx \in R^{k},\]
$c\geqslant 0$ is the Lagrangian multiplier. 

%In the next part, we are going to give our conclusion about the theoretical analysis of the efficacy of this %CW-L2 attack method.

\subsection{Task1. Increasing relative probability of the target label}
\begin{theorem}\label{Thm-7}
For the CW-L2 attack, when the Lagrangian multiplier:
\[c < \min_{j\neq y}\frac{((Wx)_{y}-(Wx)_{j})^{2}}{\|W_{y,:}-W_{j,:}\|_{2}^{2}\cdot((Wx)_{y}-(Wx)_{t})},\]
where, $W_{i, :}$ is the $i$-th row vector of $W\in R^{k\times n}$, $1\leqslant i\leqslant k$, and $y$ is the label of the original input $x$, i.e., $(Wx)_{y}$ is the largest among all the $(Wx)_{i},~(1\leqslant i\leqslant k)$. Then,
\[\frac{P(f(x)=t)}{P(f(x)=y)}\leqslant \frac{P(f(x+\delta)=t)}{P(f(x+\delta)=y)}\]
In other words, the attack increases the ratio between probability of the $t$-th and $y$-th labels.
\end{theorem}

\begin{proof}
Let $x' \doteq x + \delta$. We introduce a lemma first and then consider two different cases.
\begin{lemma}\label{Lemma2}
If $(Wx')_{y}\leqslant (Wx')_{i}$, then we have
\[\|x'-x\|_{2}^{2}>c\cdot ((Wx)_{y}-(Wx)_{t})\]
\end{lemma}
\begin{proof}
In fact, if $(Wx')_{y}\leqslant (Wx')_{i}$, then
\begin{equation}
\begin{aligned}
& (Wx)_{y}+(W\delta)_{y}\leqslant (Wx)_{i}+(W\delta)_{i}\\
\Rightarrow & (W_{y,:}-W_{i,:})^{T}\delta \leqslant -((Wx)_{y}-(Wx)_{i})\\
\Rightarrow & |(W_{y,:}-W_{i,:})^{T}\delta|\geqslant (Wx)_{y}-(Wx)_{i}
\end{aligned}
\end{equation}
According to the Cauchy-Schwarz Inequality, we have
\[|(W_{y,:}-W_{i,:})^{T}\delta|^{2}\leqslant \|W_{y,:}-W_{i,:}\|_{2}^{2}\cdot \|\delta\|_{2}^{2}\]
Therefore
\begin{equation}
\begin{aligned}
\|\delta\|_{2}^{2} &\geqslant \frac{((Wx)_{y}-(Wx)_{i})^{2}}{\|W_{y,:}-W_{i,:}\|_{2}^{2}}\\
&> c\cdot ((Wx)_{y}-(Wx)_{t})
\end{aligned}
\end{equation}
Till now, the lemma has been proved.
\end{proof}

Let us go back to the theorem, we consider the following two cases.

Case 1. $(Wx')_{y}$ is the largest among all the $(Wx')_{i},~(1\leqslant i\leqslant k)$.
\begin{equation}
\begin{aligned}
L(x') &= \|\delta\|_{2}^{2}+c\cdot g(x')\\
&= \|\delta\|_{2}^{2}+c\cdot\max(\max_{i\neq t}(Wx')_{i}-(Wx')_{t},0)\\
&= \|\delta\|_{2}^{2}+c\cdot((Wx')_{y}-(Wx')_{t})\\
&= \sum_{j=1}^{n}\delta_{j}^{2}+c\cdot\sum_{j=1}^{n}(w_{yj}-w_{tj})(x_{j}+\delta_{j})
\end{aligned}
\end{equation}
In order to minimize $L(x')$, it is easy to find that:
\[\delta_{j} = -\frac{c}{2}(w_{yj}-w_{tj}),~~1\leqslant j\leqslant n\]
On the other hand, if the equation above holds, then:
\begin{equation}
\begin{aligned}
\|\delta\|_{2}^{2} &= \frac{c^{2}}{4}\|W_{y,:}-W_{t,:}\|_{2}^{2}\\
&< c^{2}\cdot\|W_{y,:}-W_{t,:}\|_{2}^{2}\\
&< c((Wx)_{y}-(Wx)_{t})
\end{aligned}
\end{equation}\\

{\it
(The last line is because of the upper bound of $c$.
Let $j = t$, we know that:
\begin{equation}
\begin{aligned}
c &< \frac{((Wx)_{y}-(Wx)_{t})^{2}}{\|W_{y,:}-W_{i,:}\|_{2}^{2}\cdot((Wx)_{y}-(Wx)_{t})}\\
&=\frac{(Wx)_{y}-(Wx)_{t}}{\|W_{y,:}-W_{i,:}\|_{2}^{2}}
\end{aligned}
\end{equation}
That makes the inequality above proved.)
}

Then, according to Lemma.~\ref{Lemma2}, $(Wx')_{y}>(Wx')_{i}$ holds for all $i\neq y$, which is exactly the condition of Case 1. So under Case 1, the minimal value of $L$ is:
\begin{equation}
\begin{aligned}
L^{*} &= c\cdot\sum_{j=1}^{n}(w_{yj}-w_{tj})x_{j}-\frac{c^{2}}{4}\sum_{j=1}^{n}(w_{yj}-w_{tj})^{2}\\
&= c\cdot((Wx)_{y}-(Wx)_{j})-\frac{c^{2}}{4}\sum_{j=1}^{n}(w_{yj}-w_{tj})^{2}
\end{aligned}
\end{equation}

Case 2. $(Wx')_{y}$ is not the largest among all the $(Wx')_{i}~~(1\leqslant i\leqslant k)$.
Then, according to the lemma, we know that:
\[\|\delta\|_{2}^{2}>c((Wx)_{y}-(Wx)_{t})\]
Hence,
\[L(x')\geqslant \|\delta\|_{2}^{2} > c\cdot((Wx)_{y}-(Wx)_{t})\geqslant L^{*}.\]

Therefore, in this Case 2. The $x'$ is not the solution of our optimization problem.

To sum up, the accurate solution of the optimization problem is:
\[\delta_{j} = -\frac{c}{2}(w_{yj}-w_{tj})\]
Then:
\[\frac{P(f(x)=t)}{P(f(x)=y)}=\frac{{\rm softmax}(Wx)_{t}}{{\rm softmax}(Wx)_{y}}=\exp((Wx)_{t}-(Wx)_{y})\]
and the same reason applies:
\[\frac{P(f(x')=t)}{P(f(x')=y)}=\exp((Wx')_{t}-(Wx')_{y})\]
So, we only have to prove the following inequality:
\begin{equation}
\begin{aligned}
& (Wx)_{t}-(Wx)_{y}\leqslant (Wx')_{t}-(Wx')_{y}\\
\Leftrightarrow & \sum_{j=1}^{n}(w_{tj}-w_{yj})(x'_{j}-x_{j})\geqslant 0\\
\Leftrightarrow & \sum_{j=1}^{n}(w_{tj}-w_{yj})\delta_{j}\geqslant 0\\
\end{aligned}
\end{equation}
Since $\delta_{j} = -\frac{c}{2}(w_{yj}-w_{tj})$, the inequality above is obvious.
\end{proof}

\subsection{Task2. Analysis of irrelevant labels}
In this part, we argue that for any third label $k\neq y,t$, the relative probability may either increase or decrease. For $c$ satisfying the condition in Theorem.~\ref{Thm-7}, we can use the same way to prove that
\begin{equation}
\begin{aligned}
& \frac{P(f(x)=k)}{P(f(x)=y)} < \frac{P(f(x')=k)}{P(f(x')=y)}\\
\Leftrightarrow & (Wx)_{k}-(Wx)_{y}<(Wx')_{k}-(Wx')_{y}\\
\Leftrightarrow & (W_{k,:}-W_{y,:})\cdot(x'-x) > 0\\
\end{aligned}
\end{equation}
Since $x'- x = -\frac{c}{2}(W_{y,:}-W_{t,:})^{T}$, the inequality above is equivalent to:
\[(W_{k,:}-W_{y,:})\cdot(W_{y,:}-W_{t,:}) < 0.\]
And it is obvious that the inequality may be either right or wrong. For example, consider a 3D example, let $W_{y, :}=(1, 1, 1), W_{t, :}=(1, 0, 1)$. Then, if $W_{k, :} = (0, 2, 0)$, the inequality is wrong; if $W_{k, :} = (0, -1, 0)$, it is right.

\section{Numerical Results} \label{Numerical-Results}
We verify the above theoretical results by applying the aforementioned adversarial attacks to attack the ResNet56. We train ResNet56 on the CIFAR10 follow the standard procedure used by \cite{ResNet}. 
%After the neural nets is trained, we attack it by certain attack method.

\subsection{Untargeted FGSM Attack}
We first consider single and multiple iterations of the untergated FGSM attack. In Section.~\ref{Untargeted-FGSM}, we proved that for small attack strength $\epsilon$, the attack will fool the neural nets, this is validated by the numerical results shown in Fig.~\ref{fig:Untarget} (a), when $\epsilon < 0.1$, as $\epsilon$ increases, the classification accuracy decays monotonically. Moreover, Fig.~\ref{fig:Untarget} (b) shows that for a fixed $\epsilon$, as the number of iteration increases, the success rate of adversarial attack increases. This shows that IFGSM is a stronger attack than single step FGSM.

% Epsilon vs targeted attack successful rate, misclassify
\begin{figure}[h]
\centering
\begin{tabular}{cc}
\includegraphics[width=0.3\columnwidth]{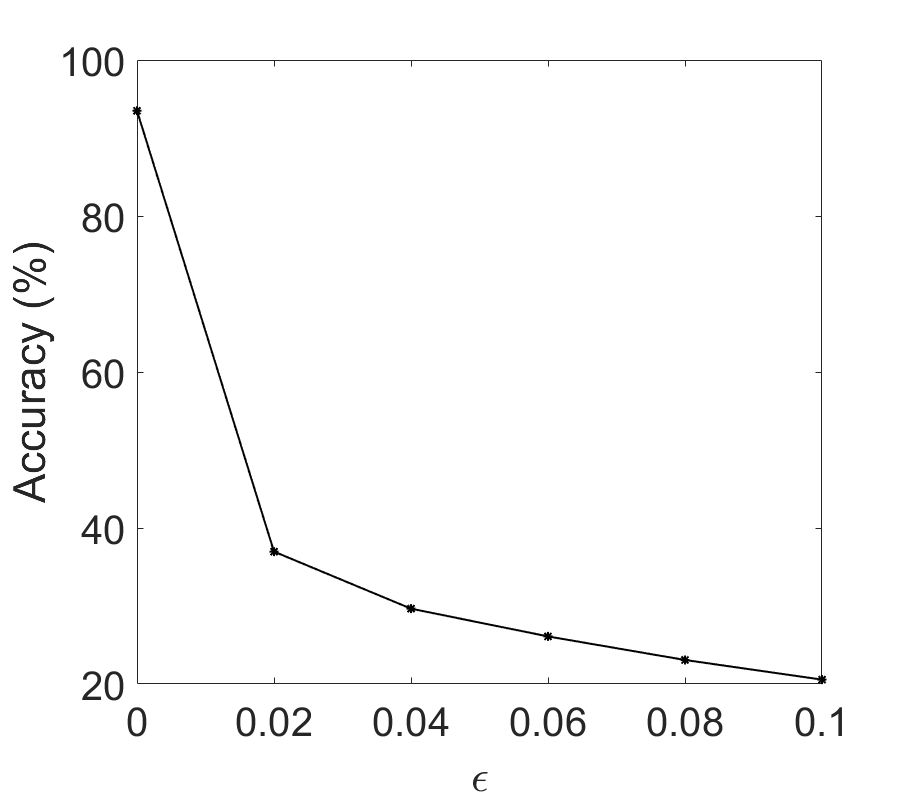} &
\includegraphics[width=0.3\columnwidth]{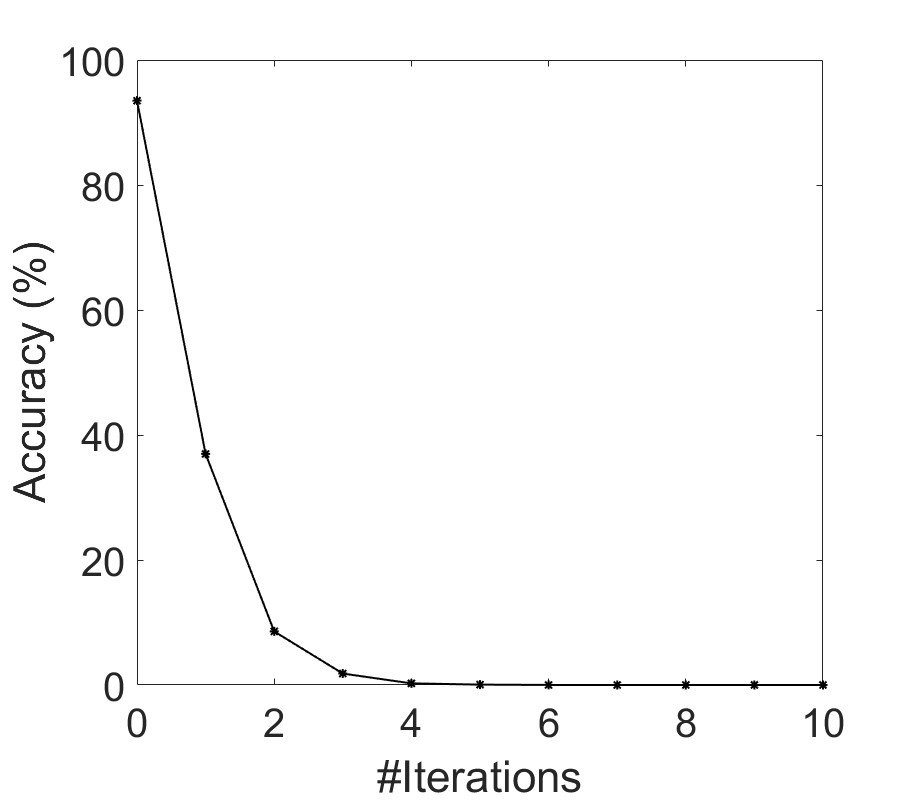}\\
(a)&(b)\\
\end{tabular}
\caption{Testing accuracy under the untargted FGSM attack for ResNet56 on CIFAR10 benchmark. (a): attack strength $\epsilon$ v.s. accuracy with only one iteration. (b): number of iterations v.s. accuracy with $\epsilon=0.02$.}
\label{fig:Untarget}
\end{figure}

% Fix epsilon, plot iters vs targeted attack successful rate, misclassify

\subsection{Targeted FGSM Attack}
% Epsilon vs targeted attack successful rate, to a given class
In this part, we verify the efficacy of the targeted FGSM attack numerically. We apply IFGSM to attack the cat (labeled 4 in CIFAR10) to dog (labeled 6 in CIFAR10). In all the experiments, we set $\alpha=0.1$. Theoretically, in Section.~\ref{Target-FGSM}, we showed that for any CNN with a softmax output activation and ReLU activation, within the regime of small $\epsilon$, targeted FGSM will fooled neural nets to classify cat to dog. Numerically, again, we consider the ResNet56. Figure.~\ref{fig:Target} (a) shows that for 10 iterations attack, when $\epsilon$ is sufficiently small ($\leq 0.015$), as $\epsilon$ increases, the success rate raises. Once $\epsilon > 0.015$, the success rate decays as $\epsilon$ increases. Furthermore, we consider impact of the number of iterations in the targeted adversarial attack. As shown in Figure.~\ref{fig:Target} (b), the success rate increases monotonically as the number of iterations increases.

\begin{figure}[h]
\centering
\begin{tabular}{cc}
\includegraphics[width=0.3\columnwidth]{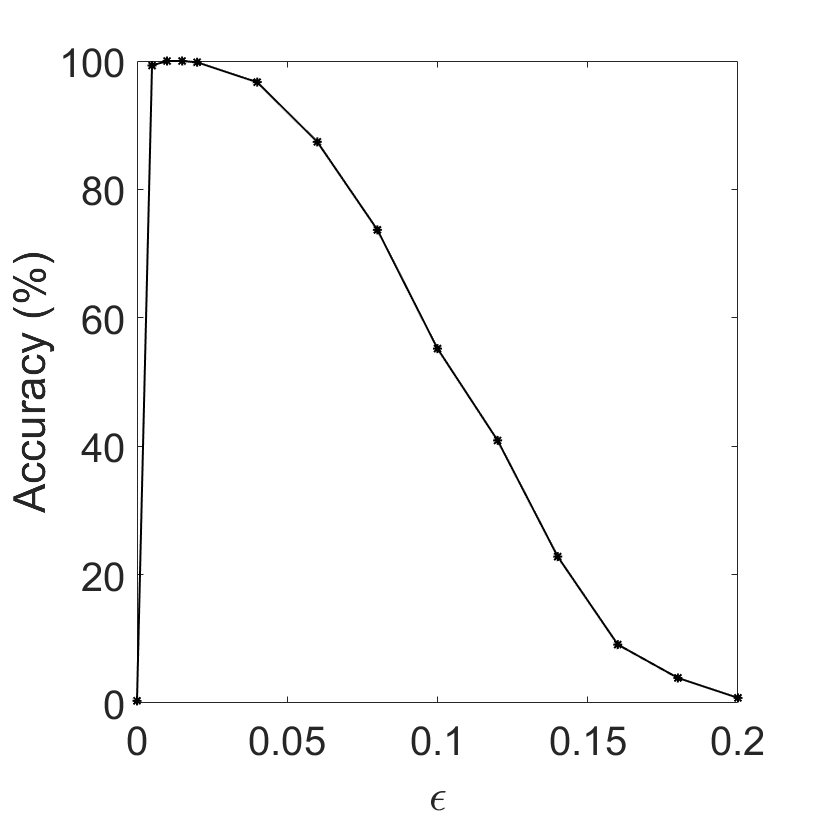} &
\includegraphics[width=0.3\columnwidth]{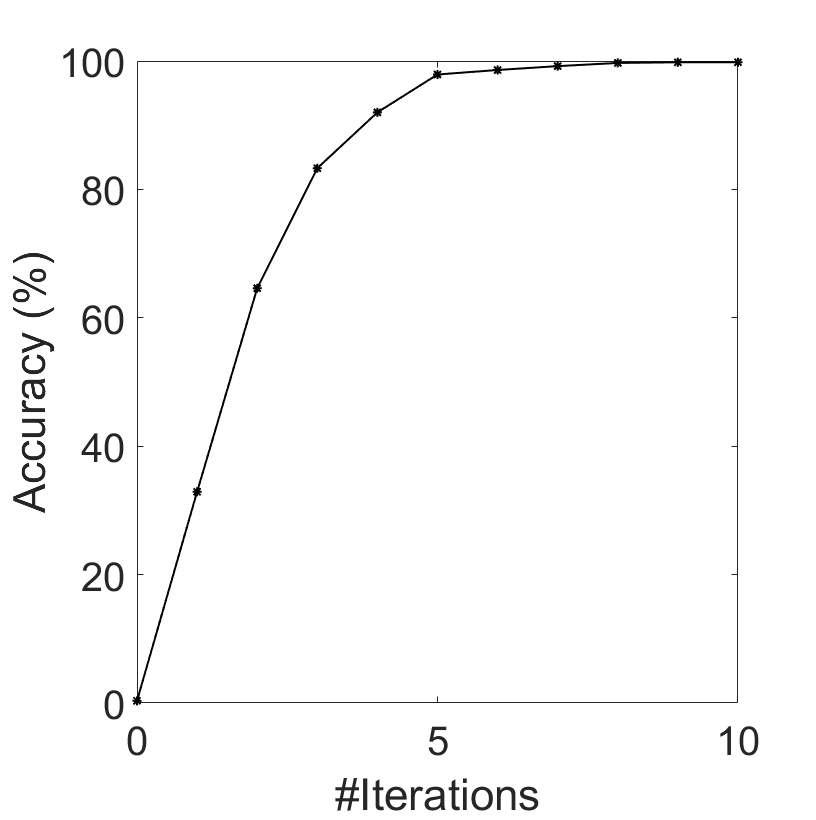}\\
(a)&(b)\\
\end{tabular}
\caption{Success rate on fooling ResNet56 to classify cat to dog under the targted FGSM attack for CIFAR10 benchmark. (a): attack strength $\epsilon$ v.s. success rate with 10 iteration. (b): number of iterations v.s. accuracy with $\epsilon=0.02$.}
\label{fig:Target}
\end{figure}

% Fix epsilon, plot iters vs targeted attack successful rate, to a given class

\subsection{CW-L2 Attack}
% Plot the probability ratio before attack and after attack
%Target class/ground truth
%Other class/ground truth
We proved, in Theorem.~\ref{Thm-7}, that when CW-L2 attack is applied to fool the neural nets to classify an image to a target class. The rate of the classification probability between the target and the ground-truth labels increases. We continue to attack ResNet56 to mis-classify cat to dog. We apply 10 iterations of Adam optimizer with, $c=10$. $\kappa=0$, and learning rate to be $0.01$ to optimize the CW-L2 attack objective (Eq.~(\ref{CWL2})). In Fig.~\ref{fig:CWL2}, we depict the averaged probability of ResNet56 to classify the cat images before and after the adversarial attack over 1000 images. 750 images successfully attacked to the dog class. From Fig.~\ref{fig:CWL2}, we see that the CW-L2 attack shift the probability density peak from class 4 (cat) to class 6 (dog). It is also interesting to notice that the probability of the images been classified to class 9 has a probability $0.0048$ before attack, while it decreases to $0.004$ after the adversarial attack. This further validates the theoretical conclusion we achieved in Section.~\ref{Target-CWL2}.
% 10 iterations of Adam optimization.

\begin{figure}[h]
\centering
\begin{tabular}{c}
\includegraphics[width=0.6\columnwidth]{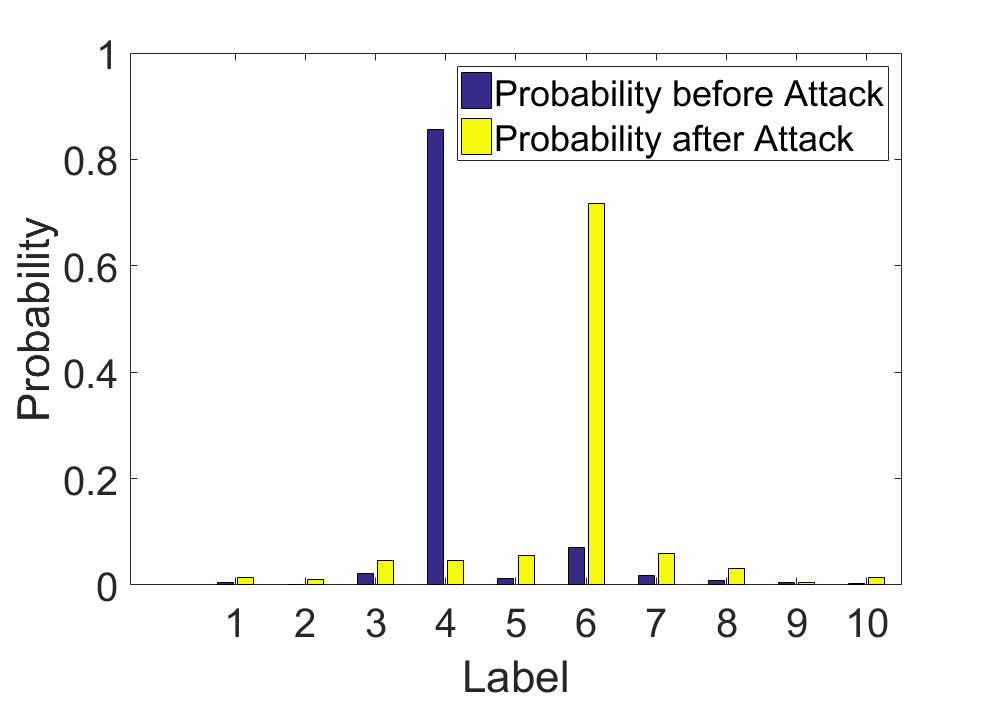}\\
\end{tabular}
\caption{The probability distribution of ResNet56 in classifying cat (labeled with 4) into each classes. Blue bars: before attack. Yellow bars: after CW-L2 attack.}
\label{fig:CWL2}
\end{figure}

\section{Concluding Remarks} \label{Conclusion}
In this paper we proved that the untargeted FGSM can fool any convolutional neural nets (CNN) with ReLU activation for small attack strength; within a specific regime, the targeted FGSM can mislead any CNN with ReLU activation to classify any given image into any prescribed class. For a two-layer neural nets, a linear layer followed by the softmax output activation, we show that the CW-L2 attack increases the ratio of the classification probability between the target and ground truth labels. A large amount of numerical results conform our theoretical results. Quantifying the relation between the attack strength $\epsilon$ and the loss is under our further exploration. Analyzing the influence of the recently proposed Laplacian smoothing gradient descent \cite{Osher:2018LSGD} on improving adversarial robustness is also under our investigation.

\section{Acknowledgments}
This material is based on research sponsored by the Air Force Research Laboratory: DARPA under agreement number FA8750-18-2-0066, MURI under the grant number FA9550-18-1-0502, and FA9550-18-1-0167. And by the U.S. Department of Energy, Office of Science, and by National Science Foundation, under grant numbers DOE-SC0013838
and DMS-1554564, (STROBE). And by the NSF DMS-1737770 and the Simons foundation. And Office of Naval Research: ONR:N00014-18-1-2527. The U.S. Government is authorized to reproduce and distribute reprints for Governmental purposes notwithstanding any copyright notation thereon.

\end{document}